\PassOptionsToPackage{algo2e}{algorithm2e}%

\documentclass[onecolumn,numeric]{miri-tech-article}

\bibliography{MIRIPublications.bib,General.bib,Inbox.bib}{}

\usepackage{miritools}

\usepackage{twoopt}
\usepackage{thm-restate}

\hypersetup{colorlinks=false}

\newcommandtwoopt{\TestSeq}[2][][n]{\ensuremath{\texttt{\upshape testseq}_{#2}\ifstrempty{#1}{}{\!\texttt{\upshape (}\text{#1}\texttt{\upshape )}}}\xspace}
\newcommandtwoopt{\RelScore}[2][][n]{\ensuremath{\texttt{\upshape relscore}_{#2}\ifstrempty{#1}{}{\!\texttt{\upshape (}\text{#1}\texttt{\upshape )}}}\xspace}
\newcommandtwoopt{\MaxScore}[2][][n]{\ensuremath{\xspace\texttt{\upshape maxscore}_{#2}\ifstrempty{#1}{}{\!\texttt{\upshape (}\text{#1}\texttt{\upshape )}}}\xspace}
\newcommand{\EvOp}[1][]{\ensuremath{\texttt{\upshape EvOp}\ifstrempty{#1}{}{\!\texttt{\upshape (}\text{#1}\texttt{\upshape )}\!}}\xspace}

\newcommand{\grad}{\ensuremath{\mathop{\nabla}}}
\newcommand{\norm}[1]{\ensuremath{\lVert#1\rVert}}

\newcommand{\ceil}[1]{\ensuremath{\lceil#1\rceil}}
\newcommand{\F}{\ensuremath{\mathcal F}\xspace}

\newcommand{\Ov}[1][t]{\ensuremath{\mathbf O_{#1}}\xspace}
\newcommand{\obs}[1][t]{\ensuremath{o_{#1}}\xspace}
\newcommand{\allobs}[1][t]{\ensuremath{o_{\prec {#1}}}\xspace}

\newcommand{\Xs}{\ensuremath{\mathcal X}\xspace}
\newcommand{\Xv}[1][i]{\ensuremath{\mathbf X_{#1}}\xspace}
\newcommand{\x}[1][i]{\ensuremath{x_{#1}}\xspace}

\newcommand{\Ys}{\ensuremath{\mathcal Y}\xspace}
\newcommand{\y}[1][]{\ensuremath{y_{#1}}\xspace}

\newcommand{\Loss}{\ensuremath{\mathcal L}\xspace}
\newcommandtwoopt{\Regret}[2][s][n]{\ensuremath{R_{#1}^{#2}}\xspace}

\newcommand{\Pbad}{\ensuremath{P^\#}\xspace}

\newcommand{\Dom}[1]{\ensuremath{\mathrm{dom}(#1)}}

\title{Asymptotic Convergence in Online Learning with Unbounded Delays}

\author{Scott Garrabrant \and Nate Soares \and Jessica Taylor \\ Machine Intelligence Research Institute \\ \{scott,nate,jessica\}@intelligence.org}

\begin{document}

\maketitle

\begin{abstract}
  We study the problem of predicting the results of computations that are too expensive to run, via the observation of the results of smaller computations. We model this as an online learning problem with delayed feedback, where the length of the delay is unbounded, which we study mainly in a stochastic setting. We show that in this setting, consistency is not possible in general, and that optimal forecasters might not have average regret going to zero. However, it is still possible to give algorithms that converge asymptotically to Bayes-optimal predictions, by evaluating forecasters on specific sparse independent subsequences of their predictions. We give an algorithm that does this, which converges asymptotically on good behavior, and give very weak bounds on how long it takes to converge. We then relate our results back to the problem of predicting large computations in a deterministic setting.

\end{abstract}

\section{Introduction}

We study the problem of predicting the results of computations that are too large to evaluate, given observation of the results of running many smaller computations. For example, we might have a physics simulator and want to predict the final location of a ball in a large environment, after observing many simulated runs of small environments.

When predicting the outputs of computations so large that they cannot be evaluated, generating training data requires a bit of creativity. Intuitively, one potential solution is this: Given enough computing resources to evaluate ``medium-sized" computations, we could train a learner by showing it many runs of small computations, and having it learn to predict the medium-sized ones, in a way that generalizes well. Then we could feed it runs of many medium-sized computations and have it predict large ones. This is an online learning problem, where the learner observes the results of more and more expensive computations, and predicts the behavior of computations that are much more difficult to evaluate than anything it has observed so far.

The standard online learning setting, in which the learner predicts an outcome in a sequence after observing all previous outcomes, does not capture this problem, because delays between prediction and observation are the key feature. \citet{Dudik:2011}, \citet{Joulani:2013}, and others have studied online learning with delayed feedback, but they assume that delays are bounded, whereas in our setting the delays necessarily grow ever-larger. In this paper, we propose an algorithm \EvOp for online learning with unbounded delays. \EvOp not a practical algorithm; it is only a first step towards modeling the problem of predicting large computations as an online learning problem.

Predicting a sequence generated by arbitrary computations is intractable in general. Consider, for instance, the bitstring that tells which Turing machines halt. However, the problem is not hopeless, either: Consider the bitstring where the $n$th digit is a 1 if and only if the $10^n$th digit in the decimal expansion of $\pi$ is a 7. This is an online learning problem with ever-growing delays where a learner should be able to perform quite well. A learner that attempts to predict the behavior of computations in full generality will encounter some subsequences that it cannot predict, but it will encounter others that are highly regular, and it should be able to identify those and predict them well.

Consider, for instance, the bitstring that interleaves information about which Turing machines halt with the $10^n$th digits of $\pi$. Intuitively, a good predictor should identify the second subsequence, and assign extreme probabilities whenever it has the computing resources to compute the digit, and roughly 10\% probability otherwise, in lieu of other information about the digit. However, it's not clear how to formalize this intuition: What does it mean for a forecaster to have no relevant information about a digit of $\pi$ that it knows how to compute? What are the ``correct" probabilities a bounded reasoner should assign to deterministic facts that it lacks the resources to compute?

In this paper, we sidestep those questions, by analyzing the problem in a stochastic setting. This lets us study the problem of picking out patterns in subsequences in the face of unbounded delays, in a setting where the ``correct" probabilities that a predictor should be assigning are well-defined. In \Sec{deterministic} we relate our findings back to the deterministic setting, making use of ``algorithmic randomness" as described by, e.g., \citet{Downey:2010}.

We propose an algorithm \EvOp with the property that, on any subsequence for which an expert that it consults predicts the true probabilities, it converges to optimal behavior on that subsequence. We show that regret and average regret are poor measures of performance in this setting, by demonstrating that in environments with unbounded delays between prediction and feedback, optimal predictors can fail to have average regret going to zero. \EvOp works around these difficulties by comparing forecasters on sparse subsequences of their predictions; this means that, while we can put bounds on how long it takes \EvOp to converge, the bounds are very, very weak. Furthermore, \EvOp is only guaranteed to converge to good behavior on subsequences when it has access to optimal experts; we leave it to future work to give a variant that can match the behavior of the best available expert even if it is non-optimal.

In \Sec{setup} we define the problem of online learning with unbounded delays. In \Sec{problem} we show that consistency is impossible and discuss other difficulties. In \Sec{solution} we define \EvOp, prove that it converges to Bayes-optimal behavior on any subsequence for which some expert makes Bayes-optimal predictions, and provide very weak bounds on how long convergence takes. In \Sec{deterministic} we relate these results back to the deterministic setting. \Sec{conclusions} concludes.

\subsection{Related Work}

An early example of online sequence learning using expert advice is \citet{Littlestone:1994}; much work has been done since then to understand how to perform well relative to a given set of forecasters~\citep{Vovk:1990,Cesa:1998,Haussler:1995}. \citet{Rakhlin:2012}~improve performance of online learning algorithms assuming some structure in the environment, while maintaining worst-case guarantees. \citet{Gofer:2013}~study the case with a potentially unbounded number of experts.

Most work in online learning has focused on the case where feedback is immediate. \citet{Piccolboni:2001}~study online prediction with less rigid feedback schemes, proving only weak performance bounds. \citet{Weinberger:2002}~show that running experts on sub-sampled sequences can give better bounds, for the case with bounded feedback delay. In the widely studied bandit setting~\citep{Auer:2002}, some attention has been given to learning with bounded delays~\citep{Neu:2010,Dudik:2011}. There have been some attempts to work with unbounded feedback delays \citep{Mesterharm:2005,Mesterharm:2007,Desautels:2014}, with either strong assumptions on the target function or with weak performance bounds. \citet{Quanrud:2015} achieve reasonable regret bounds in an adversarial setting; our work achieves asymptotic convergence in a stochastic setting. A review, and a very general framework for online learning with arbitrary (but bounded) feedback delay is given by \citet{Joulani:2013}.

Online learning with delayed feedback has applications in domains such as webpage prefetching, since the prediction algorithm has to make some prefetching decisions before learning whether a previously fetched page ended up being requested by the user \citep{Padmanabhan:1996}. The idea of learning from computations with delay has seen some use in parallel computation, e.g., distributed stochastic optimization where computations of gradients may take longer in some nodes \citep{Zinkevich:2009,Agarwal:2011}.

Outside the field of online learning, our work has interesting parallels in the field of mathematical logic. \citet{Hutter:2013} and \citet{Demski:2012a} study the problem of assigning probabilities to sentences in logic while respecting certain relationships between them, a practice that dates back to \citet{Gaifman:1964}. Because sentences in mathematical logic are expressive enough to make claims about the behavior of computations (such as ``this computation will use less memory than that one"), their work can be seen as a different approach to the problems we discuss in this paper.

\section{The Unbounded Delay Model} \label{sec:setup}

Let \Xs be a set of possible outcomes and \Ys be a set of possible predictions, where \Ys is a convex subset of $\RR^n$ for some $n$. Let $\Loss : \Xs \times \Ys \to \RR$ be a loss function measuring the difference between them, which is strongly convex (with strong convexity constant $\rho$) and Lipschitz (with Lipschitz constant $\kappa$). Roughly speaking, the environment will stochastically produce an infinite sequence of outcomes~\x[i], and an infinite sequence of observations~\obs[i], where each~\obs[i] contains information about finitely many~\x[n]. Formally, for each $i = 1, 2, \ldots,$ let $\obs[i] : \NN \to \Xs$ be a finite-domain partial function from indices to outcomes; in other words, $\obs[i]$ is a set of $(n, \x[])$ ``feedback" pairs such that each $n$ appears in at most one pair. We write $\obs[i](n)$ for the value of $x$ associated with $n$, which is feedback about the outcome $\x[n]$, and which may be undefined. If $\obs[i](n)$ is defined, we say that $\obs[i]$ reveals $\x[n]$.

Formally, we write $\Xv[i]$ for the random variable representing the $i$th output and $\Ov[i]$ for the random variable representing the $i$th observation. We define the \emph{true environment} $P$ to be a joint distribution over the $\Xv[i]$ and the $\Ov[i]$, such that if $\obs[i](n) = \x[n]$ then $P(\Ov[i]=\obs[i] \land \Xv[n]\neq\x[n]) = 0$, which means that all $\obs[i](n)$ which are defined agree on the value of $\x[n]$. We omit the random variables if we can do so unambiguously, writing, e.g., $P(\x[n] \mid \obs[i]).$

Note that there may exist~$n$ such that $\obs[i](n)$ is not defined for any $i$, in which case the forecaster will never observe $\x[n]$. We write $\allobs[i]$ for the list of observations up to time $i$, and $\allobs[i](n)$ for the value of $\x[n]$ if any observation in $\allobs[i]$ reveals it.

We consider learning algorithms that make use of some set \F of forecasters.

\begin{definition} \label{def:forecaster}
  A \textbf{forecaster} is a partial function $f$ which takes as input $n$ observations \allobs[n] and might produce a prediction $\y[n] \in \Ys$, interpreted as a prediction of $\x[n]$.
\end{definition}
Because some outcomes may never be observed, and because forecasters are partial (and so may abstain from making predictions on certain subsequences of the outcomes), we will compare forecasters only on subsequences on which both are defined.

\begin{definition} \label{def:defon}
  A \textbf{subsequence} $s$ of the outcomes is a monotonic strictly increasing list of natural numbers $s_1s_2\ldots.$ We write $|s|$ for the length of $s$, which may be $\infty$. A forecaster $f$ is \textbf{defined on} $s$ if it outputs a prediction for all elements $s_i$ of $s$, i.e., if, for all $i \le |s|$, $\y[s_i] \coloneqq f(\allobs[s_i])$ is defined.
\end{definition}
We assume that at least one $f \in \F$ is defined everywhere. It may seem prohibitively expensive to evaluate $f(\allobs[s_i])$ if $s_i$ is large. For example, consider the subsequence $s = 1, 10, 100, \ldots$; $f$ only predicts $\x[10^{10}]$ after making $10^{10}$ observations, despite the fact that $\x[10^{10}]$ is the eleventh element in the subsequence. However, there is no requirement that observations contain lots of feedback: $\allobs[s_i]$ might not reveal very much, even if $s_i$ is large.

The goal of a forecaster is to minimize its loss $\sum_{i=1}^n \Loss(\x[s_i], \y[s_i])$, for $n \ge 1$. Two forecasters can be compared by comparing their total loss.
\begin{definition} \label{def:regret}
  Given a forecaster $f$ defined on a subsequence $s$ of length at least $n$, let \begin{equation}\F_s \coloneqq \{f^\prime \in \F \mid \text{$f^\prime$ is defined on $s$}\}.\end{equation} Then the \textbf{regret} of $f$ (on $s$, through $n$) is \begin{equation}\Regret(f) \coloneqq \max_{f^\prime \in \F_s} \sum_{t=i}^n \Loss\left(\x[s_i], f(\allobs[s_i])\right) - \sum_{i=1}^n \Loss\left(\x[s_i], f^\prime(\allobs[s_i])\right).\end{equation}
  $f$ is \textbf{consistent} (with respect to $\F_s$) if its average expected regret goes to zero, that is, if \begin{equation}\lim_{n \to \infty} \sfrac{\EE[\Regret(f)]}{n}=0.\end{equation}
\end{definition}
In our setting, consistency is too strong a guarantee to ask for, as we will see in \Sec{problem}. Instead, we present an algorithm \EvOp with the property that, whenever there is a forecaster $f \in \F$ that is Bayes-optimal on some subsequence, \EvOp eventually learns to predict optimally on that subsequence.

\begin{definition} \label{def:optimal}
  A forecaster $f$ is \textbf{Bayes-optimal} (for the true environment, in its domain) if:
  \begin{enumerate}
    \item Everything $f$ predicts is almost surely eventually revealed. That is, if $f(\allobs[n])$ is defined, then with probability~1 there is some $N$ such that $\obs[N](n)$ is defined.
    \item $f$ minimizes expected loss against the true environment whenever it makes a prediction. That is, if $\y[n] \coloneqq f(\allobs[n])$ is defined, then $\y[n] = \argmin_{\y} \EE[\Loss(\x[n],\y)\mid \allobs[n]].$
  \end{enumerate}
\end{definition}

We will occasionally refer to a Bayes-optimal $f$ as simply ``optimal".

The main result of our paper is this: Whenever there is an optimal forecaster $f \in F$ defined on $s$, our algorithm \EvOp converges to optimal behavior on $s$.
\begin{restatable}{theorem}{thmevop} \label{thm:evop}
  For any Bayes-optimal $f^s \in \F$ defined on $s$,
  \begin{equation}\lim_{n \to \infty} |\Loss\left(\x[s_n], \EvOp[{\allobs[s_n]}]\right) - \Loss\left(\x[s_n], f^s(\allobs[s_n])\right)| = 0.\end{equation}
\end{restatable}
\noindent We call algorithms with this property \emph{eventually optimal}.
We will define \EvOp in \Sec{solution}, and prove \Thm{evop} in \Sec{proof}. Weak bounds on how long it takes \EvOp to converge to Bayes-optimal behavior on any individual subsequence are given in \Sec{bounds}.

Eventual optimality is a very strong condition, and only yields guarantees if \F contains Bayes-optimal forecasters. In this paper we focus on showing that an eventually optimal predictor exists, and providing weak bounds on how long it takes it to converge to optimal behavior on a subsequence (and how much loss can be accumulated in the meantime). As we will see in \Sec{problem}, this is non-trivial. We leave the problem of converging on the best available forecaster of a subsequence (even if it is not optimal) to future research.

\section{Difficulties in this Setting} \label{sec:problem}

Total regret and average regret are poor measures of forecaster performance in this setting, and consistency (as defined by \Def{regret}) is impossible in general. To show this, we will describe an environment \Pbad which exploits the long delays to make learning difficult.

\Pbad generates outcomes as follows. It flips a fair coin and reveals it once, and then flips another and reveals it ten times, then flips a third and reveals it one hundred times, and so on, always revealing the $k$th coin $10^{k-1}$ times. The forecasters spend one timestep predicting the first coin, ten timesteps predicting the second coin, one hundred timesteps predicting the third coin, and so on. The observations are set up such that they contain no information about the coin currently being predicted: The forecasters must predict the $k$th coin all $10^{k-1}$ times before it is revealed.

Formally, let $\Xs \coloneqq \Set{\textsc{h}, \textsc{t}}$ corresponding to ``heads" and ``tails" respectively. Let \Ys be the set of probability distributions over \Xs, which can be represented as real number ${p \in [0, 1]}$. \Pbad~is a Markov chain, where each \x[i+1] is conditionally independent from all other outcomes given \x[i]. ${\Pbad(\Xv[1]=\textsc{h})=0.5}.$ For ${i = 2, 12, 112, 1112, \ldots}$, $\x[i]$ ``reveals a new coin" and is independent of $\x[i-1]$: ${\Pbad(\Xv[i]=\textsc{h}\mid\Xv[i-1]=\cdot\;)=0.5}$. For all other $i$, $\x[i]$ ``reveals the same coin again:" $\x[i]=\x[i-1]$. Each \Ov[n] is a deterministic function of $\Xv[1]\ldots\Xv[n]$ which reveals the first $\ceil{\log_{10} \left(n \cdot \sfrac{9}{10}\right)}$ outcomes. Let \Loss be squared error; that is, let $\Loss(\textsc{h}, p) = (1-p)^2$ and $\Loss(\textsc{t}, p) = p^2$.

Clearly, the best prediction of \x[n] that a forecaster can make given \allobs[n] is $0.5$, because \allobs[n] does not contain any information about the coin revealed by \x[n], which is fair. Thus, the simple forecaster $f^*(\allobs[n]) = 0.5$ is Bayes-optimal. However, the regret of $f^*$ may be very high! To see this, consider a forecaster $f^1$, the ``gambler," defined $f^1(\allobs[n]) = 1.$ In expectation, $f^1$ will receive higher total loss on any subsequence of the true outcomes. However, $f^1$ will spend about half the time with a lower total loss than $f^*$, because each time a new coin begins being predicted, it has the opportunity to recoup all its losses.

$f^*$ accumulates loss at a rate of $\sfrac{1}{4}$ units per prediction, which means that, after the $k$th coin has been predicted all $10^{k-1}$ times, its aggregate loss is $\sfrac{1}{4} \cdot \sum_{i=1}^k 10^{i-1}$. $f^1$ accumulates either $0$ or $1$ unit of loss in each step according to whether the coin comes up heads or tails, so in the worst case, it will have $\sum_{i=1}^k 10^{i-1}$ total loss after the $k$th coin. If the $k+1$ coin comes up heads, then $f^*$ gains an additional $\sfrac{1}{4} 10^k$ loss while $f^1$'s loss remains unchanged. $10^k$ accounts for more than nine tenths of $\sum_{i=1}^k 10^i$, so if the coin came up heads then $f^1$'s total loss is at most a tenth of $\sum_{i=1}^k 10^i$, whereas $f^*$'s total loss is a quarter of $\sum_{i=1}^k 10^i$. In fact, any predictor that assigns average probability $\le 0.5$ across all $10^{k-1}$ reveals of the $k$th coin will have at least 15\% more loss than $f^1$ after the \smash{$\sum_{i=1}^k$}th step, if that coin comes up heads.

By a similar logic, whenever the $k$th coin comes up tails, $f^1$'s loss shoots up above that of $f^*$, no matter how lucky it was previously. Thus we see that if $f^1 \in \F$, the regret of $f^*$ will swing wildly back and forth. Any predictor which is maintaining a mixture of forecasters and weighting them according to their regret will have trouble singling out $f^*$.

Indeed, if the environment is \Pbad, and if $\F$ contains both $f^1$ and the opposite gambler $f^0$ defined as $f^0(\allobs[n])=0$, then it is impossible for a forecaster to be consistent in the sense of \Def{regret}. If the average probability a forecaster assigns to the $k$th coin is $\le 0.5$ and the coin comes up heads, it gets very high regret relative to $f^1$, whereas if it's $\ge 0.5$ and the coin comes up tails, it gets very high regret relative to $f^0$. The only way for a forecaster to avoid high regret against both gamblers is for it to place higher probability on the true result of the coin every single time. With probability~1 it must slip up infinitely often (because the coins are fair), so each forecaster's regret will be high infinitely often. And the amount of regret---at least 15\% of all possible loss---is proportional to~$n$, so $\lim_{n\to\infty} \sfrac{\EE[\Regret[][n](f)]}{n}$ cannot go to zero.

Lest this seem like a peculiarity of the stochastic setting, observe that a similar problem could easily occur in the deterministic setting, when a learner is predicting the behavior of large computations. For example, imagine that the ``coins" are chaotic subsystems inside a physics simulation, such that large environments have many correlated subsystems. In this case, some experts might start ``gambling" by making extreme predictions about those subsystems, and it may become difficult to distinguish the accurate forecasters from the gamblers, while looking at total or average regret.

The first fix that comes to mind is to design a predictor with a learning rate that decays over time. For example, if the learner weights the loss on $\x[n]$ by $\sfrac{1}{10n}$ then it will assign each cluster of $10^{k-1}$ predictions roughly equal weight, thereby neutralizing the gamblers. However, this fix is highly unsatisfactory: It runs into exactly the failures described above on the environment $\Pbad_2$ which reveals the $k$th coin \smash{$10^{10^k}$} times instead. It might be the case that for each specific environment one could tailor a learning rate to that environment that allows a predictor to successfully distinguish the optimal forecasters from the gamblers using regret, but this would be an ad-hockery tantamount to hardcoding the optimal forecaster in from the beginning. This motivates the study of how a predictor can successfully identify optimal experts at all in this setting.

\section{The \texorpdfstring{\EvOp}{EvOp} Algorithm} \label{sec:solution}

\Sec{problem} showed that in this setting, it is possible for gamblers to take advantage of correlated outputs and unbounded delays to achieve drastic swings in their total loss, which makes total and average regret bad measures of a forecaster. We can address this problem by comparing forecasters only on \emph{independent} subsequences of outcomes on which they are both defined.

Intuitively, the gamblers are abusing the fact that they can correlate many predictions before any feedback on those predictions is received, so we can foil the gamblers by assessing them only on a subsequence of predictions where each prediction in the subsequence was made only after receiving feedback on the previous prediction in the subsequence. \EvOp is an algorithm which makes use of this intuition, and \Thm{evop} shows that it is sufficient to allow \EvOp to zero in on Bayes-optimal predictors regardless of what strategies other forecasters in \F use.
\begin{definition} \label{def:independent}
  A sequence $s$ is \textbf{independent} if, for all $i > 1$, $\obs[s_i](s_{i-1})$ is defined.
\end{definition}
\begin{algorithm2e}
  \caption{\EvOp, an eventually optimal predictor. \norm{\cdot} is the $l^2$ norm, and $1/0 = \infty$.\label{alg:evop}}
  \SetKwData{Waiting}{waiting}
  \KwIn{$\allobs[n]$, the first $n$ observations}
  \KwData{$\varepsilon$, an arbitrary constant $< 1$}
  \BlankLine

  \tcp{Computes an independent subsequence on which $f_i$ and $f_j$ disagree.}
  \Fn{\TestSeq[$i$, $j$, $m$]}{
    $t \leftarrow 0$\;
    $\Waiting \leftarrow false$\;
    \For{$k$ in $1, 2, \ldots, n$}{
      \uIf{\Waiting and $t \in \Dom{\obs[k]}$}{
        \Output{$k$}\;
        $\Waiting \leftarrow false$\;
      }
      \ElseIf{$\y[k]^i \coloneqq f_i(\allobs[k])$ and $\y[k]^j \coloneqq f_j(\allobs[k])$ are defined, and $\norm{\y[k]^i - \y[k]^j} > 1/m$}{
        $t \leftarrow k$\;
        $\Waiting \leftarrow true$\;
      }
    }
  }
  \BlankLine

  \tcp{Computes the difference between the scores of $f_i$ and $f_j$ on an independent subsequence on which they disagree.}
  \Fn{\RelScore[$i$, $j$, $m$]}{
    $s \leftarrow \TestSeq[$i$, $j$, $m$]$\;
    \KwRet{$\sum_{k=1}^{|s|} \left(\Loss(\x[s_k], \y[s_k]^i) - \Loss(\x[s_k], \y[s_k]^j) - \dfrac{\rho\varepsilon}{2 m^2}\right)$}\;
  }
  \BlankLine

  \Fn{\MaxScore[$i$]}{
    \KwRet{$\max_{j \in \mathbb{N}_{\geq 1}, m \in \mathbb{N}_{\geq0}} i - j - m + \RelScore[$i$, $j$, $m$]$}\;
  }
  \BlankLine
  $f \leftarrow \min_{i \in \mathbb{N}_{\geq 1} \text{ such that } \allobs[n] \in \Dom{f_i}} \,\MaxScore[$i$]$\;
  \KwRet{$f(\allobs[n])$}\;
\end{algorithm2e}

\EvOp works as follows. Fix an enumeration $f_1, f_2, \ldots$ of \F, which must be countable but need not be finite; we can assume without loss of generality that this enumeration is countably infinite. \EvOp compares $f_i$ to $f_j$ by giving it a relative score, which is dependent on the difference between their loss measured only on an independent subsequence of predictions on which they are both defined, constructed greedily. Lower scores are better for $f_i$. The score is also dependent on $\rho$, the strong convexity constant for \Loss, and an arbitrary positive $\varepsilon < 1$, which we use to ensure that if $f_i$ and $f_j$ make different predictions infinitely often then their scores actually diverge. \EvOp follows the prediction of the $f_i$ chosen by minimaxing this score, i.e., it copies the $f_i$ that has the smallest worst-case score relative to any other $f_j$. Pseudocode for \EvOp is given by \Alg{evop}.

To see that the $\max$ step terminates, note that it can be computed by checking only finitely many $j$ and $m$: \RelScore[$i$, $j$, $m$] is bounded above by \smash{$\sum_{k=1}^{|s|} \Loss(\x[k], \y[k]^i)$}, so all $(j, m)$ pairs such that $j+m$ is greater than this value may be discarded. To see that the $\min$ step terminates, note that it can be computed by checking only finitely many $i$ (assuming that at least one $f$ is defined on $\allobs[n]$), because when $m=0$, \TestSeq[$i$, $j$, $m$] is empty; thus when $j=1$ and $m=0$, \MaxScore[$i$, $j$, $m$] is at least $i - 1$. Therefore, after finding the smallest $k$ such that $f_k$ is defined on \allobs[n], the $\min$ step need only continue searching up through $i = \MaxScore[$k$] + 1.$

\EvOp gets around the problems of \Sec{problem} by comparing forecasters only on greedily-constructed independent subsequences of the outcomes. Note that if the delay between prediction and feedback grows quickly, these subsequences might be very sparse. For example, in the environment \Pbad of \Sec{problem}, the independent subsequence will have at least $10^i$ timesteps between the $i-1$st element in the subsequence and the next. This technique allows \EvOp to converge on Bayes-optimal behavior, but it also means that it may do so very slowly (if the subsequence is very sparse). Under certain assumptions about the speed with which delays grow and the frequency with which forecasters disagree, it is possible to put bounds on how quickly \EvOp converges on Bayes-optimal behavior, as discussed in \Sec{bounds}. However, these bounds are quite weak.

\subsection{Proof of \Thm{evop}} \label{sec:proof}

To prove \Thm{evop} we need two lemmas, which, roughly speaking, say that (1) if $f_z$ is Bayes-optimal then \MaxScore[$z$] is bounded; and (2) if $f_j$ is not Bayes-optimal and some $f_z \in \F$ is, then \MaxScore[$j$] goes to infinity. From there, the proof is easy.

In what follows, let $f_z$ be a Bayes-optimal forecaster (as per \Def{optimal}) that makes infinitely many predictions all of which are almost surely eventually revealed---that is, such that $f_z(\allobs[n])$ is almost surely defined infinitely often, and whenever it is defined, $\obs[i](n)$ is almost surely defined for some $i$. Let $z$ be the index of $f_z$ in the enumeration over \F. In general, we will write $\y[n]^i$ for $f_i(\allobs[n])$ when it is defined.

\begin{lemma} \label{lem:bounded}
  If $f_z$ is Bayes-optimal and makes infinitely many predictions all of which are almost surely eventually revealed, then with probability~1, \MaxScore[z] is bounded.
\end{lemma}
\begin{proof} For all $j$, $\RelScore[$z$, $j$, $0$]=0$, because \TestSeq[$z$, $j$, $0$] never outputs. Thus, \MaxScore[$z$] is bounded below by $z-1$ (consider the case where $j=1$ and $m=0$) and bounded above by $z-j-m+\RelScore[$z$, $j$, $m$]$. When $m=0$ this is bounded above by $z-j$, so it suffices to show that there is almost surely some bound B such that $\RelScore[$z$, $j$, $m$]-j-m$ is bounded above by B for every $j$ and $m \ge 1$.

Intuitively, in expectation, \RelScore[$z$, $j$, $m$] should either be finite or diverge to $-\infty$, because $f_z$ is Bayes-optimal and is only being compared to other forecasters on independent subsequences. We will prove not only that it's bounded above in expectation, but that it is bounded above with probability~1. To do this we use \Lem{jessica} in \App{jessica}, which (roughly speaking) says that something which is zero in expectation, and which has ``not too much" variance in expectation, can't get too far from zero in fact.

Fix $j$, $m \geq 1$, and $\lambda$; we will bound the probability that $\RelScore[$z$, $j$, $m$] \ge \lambda.$ Let $s=s_1s_2\ldots$ be the outputs of \TestSeq[$x$, $j$, $m$][\infty], that is, the entire greedily-generated sparse independent subsequence of outputs on which both $f_z$ and $f_j$ make predictions that differ by at least $\sfrac{1}{m}$ (which could be generated by running \TestSeq[$x$,$j$,$m$] on larger and larger $n$). $s$ may or may not be finite.

Because \Loss is strongly convex,
\begin{equation}
\Loss(\x[k], \y[k]^j) \ge \Loss(\x[k], \y[k]^z)
  + \grad_{\y} \Loss(\x[k], \y[k]^z) \cdot (\y[k]^j - \y[k]^z)
  + \frac{\rho}{2} \norm{\y[k]^j - \y[k]^z}^2,
\end{equation}
where $\grad_{\y}$ takes the gradient of \Loss with respect to the prediction, $\rho$ is the strong convexity constant of \Loss, and $\norm{\cdot}$ is the $l^2$ norm. In other words, the loss of $f_j$ in any given round is at least that of $f_z$ plus a linear term (which, note, is related to the Lipschitz constant of \Loss) plus a quadratic term. Rearranging this inequality,
\begin{equation}\label{eq:l}
  \Loss(\x[k], \y[k]^z) - \Loss(\x[k], \y[k]^j) \le - \grad_{\y} \Loss(\x[k], \y[k]^z) \cdot (\y[k]^j - \y[k]^z) - \frac{\rho}{2} \norm{\y[k]^j - \y[k]^z}^2.
\end{equation}
We will show that the sum of the right-hand side for $k=1,2,\ldots,n$ is bounded, using \Lem{jessica}.

\Lem{jessica} requires a sequence of random variables $G_1H_1G_2H_2\ldots$ that form a Markov chain, and two real-valued functions $v$ and $r$ defined on the $G_i$ and the $H_i$ respectively, such that ${\EE[r(H_i)\mid v(G_i)]=0},$ and $|r(H_i)| \leq a\sqrt{v(G_i)}$ for some constant $a$. Intuitively, these constraints say that $r$ is zero in expectation, and that its absolute value is bounded by $v$. \Lem{jessica} then gives us a bound on the probability that $\sum_{i=1}^n r(H_i)-v(G_i) \ge \lambda.$ We use it with $r$ as the first term on the right-hand side of \Eqn{l}, and $v$ as the negative of the second. Roughly, $r$ can be thought of as a first-order approximation to the amount by which $f_j$ did better than expected (a ``residual"), and $v$ as a bound on how wildly $r$ can swing (a ``variance").

Let $G_i$ be $\allobs[s_i]$\footnote{This is somewhat ill-defined, since $s_i$ is itself a random variable.  We can make this more precise by defining $G_i = (s_i, \allobs[s_i])$ and noting that $s_i$ can be determined from knowing only $\allobs[s_i]$} and $H_i$ be \allobs[k] where $k$ is the least time after $s_i$ such that $s_i \in \Dom{\obs[k]}$. $k$~exists, because $f_z$ only makes predictions that, with probability~1, are eventually revealed. Intuitively, our Markov chain alternates between elements of $s$ and the times when those elements were revealed. For~$i > |s|$, let $G_i=H_i=\allobs[\infty]$, the (infinite) combination of all observations.

  Define $r$ to be the function $r(H_i) = - \grad_{\y} \Loss(\x[s_i], \y[s_i]^z) \cdot (\y[s_i]^j - \y[s_i]^z)$ when $i\leq|s|$, and 0 otherwise. Observe that this value can be calculated from $H_i$, $f_z$, and $f_j$, because $H_i=\allobs[k]$, with $k > s_i$ and $\x[s_i]\coloneqq\allobs[k](s_i)$ defined.

  Define $v$ to be the function $v(G_i) = \frac{\rho}{2} \norm{\y[s_i]^j - \y[s_i]^z}^2$ when $i\leq |s|$, and $\sfrac{\rho}{2m^2}$ otherwise, which can be calculated from~$G_i$,~$f_z$, and~$f_j$, because $G_i$ is just $\allobs[s_i]$. Note that ${\EE[r(H_k) \mid G_k]=0},$ because $f_z$ is a Bayes-optimal predictor, which means it minimizes expected loss, making the gradient in $r(H_i)$ zero in expectation for all $i$. Note also that because~\Loss is Lipschitz, ${|r(H_k)| \leq \kappa \norm{\y[n]^j - \y[n]^z}}$ where~$\kappa$ is the Lipschitz constant of~\Loss. Thus, with $a=\frac{\kappa\sqrt{2}}{\sqrt{\rho}}$, $|r(H_k)| \leq a\sqrt{v(G_k)}.$  Therefore, $r$ and $v$ meet the conditions of \Lem{jessica}, so for all~$M,$
\begin{equation} \label{eq:lambda}
\PP\left(\sum_{i=1}^n r(H_i)-v(G_i)\geq M\right)\leq\exp \left(-\rho\kappa^{-2}M\right),
\end{equation}
which goes to $0$ as $M$ goes to infinity. We need a bound that forces it to $0$ as $n \to \infty$. In what follows, we write $b=\rho\kappa^{-2}$ for conciseness.

Observe that $\RelScore[$z$,$j$,$m$]\leq\sum_{i=1}^{t_n} \left(r(H_i)-v(G_i)-\sfrac{\rho\varepsilon}{2m^2}\right)$, where $t_n$ is the number of times $\TestSeq[$z$, $j$, $m$]$ outputs. Thus, the probability that $\RelScore[$z$,$j$,$m$]\geq\Lambda$ for any $\Lambda$ is upper-bounded by the probability that, for some $n$,
\begin{equation}\left(\sum_{i=1}^{t_n} r(H_i)-v(G_i)\right)-t_n\frac{\rho\varepsilon}{2m^2}\geq\Lambda.\end{equation}
For any given $n$ and $t$, applying inequality~\eq{lambda} with $\Lambda+t\frac{\rho\varepsilon}{2m^2}$ for $M,$
\begin{equation} \label{eq:tbound}
  \PP\left(\sum_{i=1}^{t_n} r(H_i)-v(G_i)\ge\Lambda + t\frac{\rho\varepsilon}{2m^2}\right)\leq \exp\left({-b\left(\Lambda+\frac{t\rho\varepsilon}{2m^2} \right)}\right).
\end{equation}
We now see the function of the $\sfrac{\rho\varepsilon}{2m^2}$ term in \RelScore[][]: it adds a tiny bias in favor of the forecaster being judged, such that the longer a contender waits to prove itself, the more it has to prove. Equation~\eq{tbound} says that, because $f_j$ never proves itself too much in expectation, the probability that $f_z$'s score relative to $f_j$ goes strongly in $f_j$'s favor gets lower as $t_n$ gets larger.

Note that $\RelScore[$z$,$j$,$m$]$ only depends on $n$ through $t_n$: If $t_{n_1}=t_{n_2}$ for some $n_1$ and $n_2$ then $\RelScore[$z$,$j$,$m$][n_1]=\RelScore[$z$,$j$,$m$][n_2].$ Thus, $\PP\left(\exists n \colon \RelScore[$z$,$j$,$m$] > \Lambda\right)$ can be bounded by summing only over the possible values $t$ of $t_n$.
\begin{equation}\sum_{t=0}^\infty \exp\left({-b\Lambda}+t\left(-\frac{b\rho\varepsilon}{2 m^{2}}\right)\right)=\frac{\exp(-b\Lambda)}{1-\exp\left(-\frac{b\rho\varepsilon}{2m^{2}}\right)},\end{equation}
and $m\geq 1$, so
\begin{equation} \label{eq:relbound}
  \PP\big(\exists n \colon \RelScore[$z$,$j$,$m$]\ge \Lambda\big) \le \frac{\exp(-b\Lambda)}{1-\exp(-b\rho\varepsilon/2)}.
\end{equation}

Applying inequality~\eq{relbound} with $\lambda + m + j$ for $\Lambda$, we see that the probability $\RelScore[$z$,$j$,$m$]\geq \lambda+m+j$ is at most
\begin{equation}
  \sum_{m=1}^\infty\sum_{j=1}^\infty\frac{\exp(-b(\lambda+m+j))}{1-\exp(-b\rho\varepsilon/2)}=\frac{\exp(-b(\lambda+2))}{(1-\exp(-b\rho\varepsilon/2))(1-\exp(-b))^2}.
\end{equation}
This goes to 0 as $\lambda$ goes to $\infty$. Therefore, with probability~1, there exists some bound $B$ such that $\RelScore[$z$,$j$,$m$]-m-j < B$ for all $j$ and $m \ge 1$. Thus, \MaxScore[$z$] is almost surely bounded.
\end{proof}

\begin{restatable}{lemma}{leminf}\label{lem:infinite}
  If $f_z$ is Bayes-optimal and makes infinitely many predictions all of which are almost surely eventually revealed, then for any $f_j$, with probability 1, if $\y[i]^z \coloneqq f_z(\allobs[i])$ and $\y[i]^j \coloneqq f_j(\allobs[i])$ are both defined on the same $t$ infinitely often, and if $\norm{\y[i]^z - \y[i]^j} \ge \delta$ infinitely often for some $\delta > 0$, \begin{equation}\lim_{n \to \infty} \MaxScore[$j$] = \infty.\end{equation}
\end{restatable}
\noindent Roughly speaking, the proof runs as follows. Choose $m$ such that $\sfrac{1}{m} < \delta$. It suffices to show that $\RelScore[$j$,$z$,$m$] \to \infty$ as $n \to \infty$. The $\sum (\Loss(\x[k], \y[k]^j) - \Loss(\x[k], \y[k]^z))$ portion goes to infinity in expectation, and also goes to infinity with probability~1 by \Lem{jessica}. It remains to show that the $\sum\sfrac{\rho\varepsilon}{2m^2}$ terms working in $f_j$'s favor are not sufficient to prevent the total from going to infinity, which can be done by showing that the differences between \smash{$\Loss(\x[k],\y[k]^j)$ and $\Loss(\x[k],\y[k]^j)$} are at least $\sfrac{\rho}{2m^2} > \sfrac{\rho\varepsilon}{2m^2}$ in expectation, and appealing again to \Lem{jessica}. The proof proceeds similarly to the proof of \Lem{bounded}, so we leave the details to \App{infinite}.

With these lemmas in place, we now prove that \EvOp is eventually optimal. Recall \Thm{evop}:
\thmevop*

\begin{proof}
  Let $f_z$ be Bayes-optimal and defined infinitely often, such that everything it predicts is almost surely eventually revealed. It suffices to show that, with probability~1, if $f_z(\allobs[n])$ is defined then \begin{equation}\lim_{n \to \infty} \norm{\EvOp[{\allobs[n]}]-f_z(\allobs[n])} = 0.\end{equation}
  By \Lem{bounded}, \MaxScore[$z$] is bounded with probability~1. Let $B$ be this bound. Note that there are only finitely many~$i$ such that $\MaxScore[$i$] \le B$, for the same reason that the $\min$ step always terminates. For each of those $i$, either $f_i$ and $f_z$ converge to the same prediction, or they only make finitely many predictions in common, or (by \Lem{infinite}) $\MaxScore[$i$] \to \infty.$ The latter contradicts the assumption that $\MaxScore[$i$] \le B.$ If $f_i$ and $f_z$ only make finitely many predictions in common, then for sufficiently large $n$, $f_i$ is not defined and so will not be selected. Thus, we need only consider the case where $f_i$ and $f_z$ converge to the same predictions whenever they both make predictions. In this case, \EvOp[{\allobs[n]}] is choosing among finitely many forecasters all of which converge to $f_z(\allobs[n])$, so \EvOp[{\allobs[n]}] must converge to $f_z(\allobs[n])$.
\end{proof}

\subsection{Bounds} \label{sec:bounds}

The speed with which \EvOp converges to optimal behavior on a subsequence depends on both (1) the sparseness of independent subsequences in the outcomes; and (2) the frequency with which forecasters make claims that differ.

Specifically, assume that all forecasters are defined everywhere and disagree infinitely often, and that $\F$ is finite. (The first two constraints imply the third.) We can show that, given a (potentially fast-growing) function $h$ bounding how long it takes before predictors disagree with each other, and given another (potentially fast-growing) function $g$ bounding the delay in feedback, and given a probability $p$, the time it takes before $\EvOp$ has converged on $f_z$ with probability $p$ is proportional to $h\circ g$ iterated a number of times proportional to $\log p.$ (Note that $h$ and $g$ are not uniform bounds; $g(n)$ is the maximum delay between the $n$th prediction and feedback on the $n$th prediction, and delays may grow ever larger as $n$ increases.)

\begin{restatable}{theorem}{thmbounds} \label{thm:bounds}
  Given $h$, $g$, a Bayes-optimal $f_z$, and a probability $p$, there is an $N \propto (h \circ g)^{\log p}(1)$ such that, with probability at least $1-p$, for all $n \ge N$, \begin{equation}\EvOp[{\allobs[n]}] = f_z(\allobs[n]).\end{equation}
\end{restatable}
\noindent We prove \Thm{bounds} in \App{bounds}.

To call these bounds ``weak" is an understatement. In the case where the outcomes are generated by running a universal Turing machine $U$ on different inputs, $g$ is infinite, because $U$ will sometimes fail to output. It is possible to achieve \emph{much} better bounds given certain simplifying assumptions, such as delays that are finite in expectation \citep{Joulani:2013}. However, it is not yet clear which simplifying assumptions to use, or what bounds to ask for, in the setting with ever-growing delays.

\section{The Deterministic Setting} \label{sec:deterministic}

Our motivation for studying online learning with unbounded delays in a stochastic setting is that this gives us a simplified model of the problem of predicting large computations from observations of smaller ones. We have already seen one instance of an issue in the stochastic setting which looks likely to have an analog in the deterministic setting. In \Sec{problem} we gave the example of a deterministic ``coin" that appears more and more often in larger and larger computations, which might (for instance) be a common subsystem in the environment of a physical simulation. Intuitively, if there are many correlated subsystems that appear ``sufficiently random" to all forecasters, then forecasters might follow the strategies of $f^1$ and $f^0$ in \Sec{problem} and achieve regular large swings in their total loss. Intuitively, the techniques used in \Alg{evop} to handle the problem in the stochastic case should well carry over to the deterministic case, but any attempt to formalize this intuition depends on what it means for a deterministic sequence to be ``sufficiently random."

For that we turn to algorithmic information theory, a field founded by \citet{Martin:1966} which studies the degree and extent to which fixed bitstrings can be called ``random." In their canonical text, \citet{Downey:2010} give three different definitions of algorithmic randomness and show them all to be equivalent. The oldest of the three, given by \citet{Martin:1966}, is rooted in the idea that an algorithmically random sequence should satisfy all computably verifiable properties that hold with probability~1 on randomly generated sequences.

It is with this definition in mind that we note that \Lem{bounded} and \Lem{infinite} are both stated as properties that are true of randomly generated sequences with probability~1. \Lem{bounded} says that if the outputs of the environment are generated randomly, then with probability~1, the score of a Bayes-optimal predictor does not go to infinity. \Lem{infinite} says that if the outputs of the environment are generated randomly, then with probability~1, a predictor that disagrees by $\delta > 0$ with a Bayes-optimal predictor infinitely many times has its score going to infinity. Both these computable properties hold for random sequences with probability~1, so they hold for Martin-L\"{o}f-random sequences.

This means that if $\F$ is the class of all Turing machines, and \EvOp is predicting an algorithmically random sequence (such as Chaitin's $\Omega$, the fraction of Turing machines which halt), then \Thm{evop} holds and \EvOp will converge on optimal predictions on subsequences of that sequence. However, this does us no good: There are no computable patterns in Chaitin's $\Omega$; computable forecasters won't be able to do any better than predicting a 50\% chance of a 1. Besides, the goal is not to predict uncomputable sequences by running all Turing machines. The goal is to predict large computations using efficient (e.g., polynomial-time) experts.

What we need is a notion of algorithmic randomness \emph{with respect to a restricted class of experts}. For example, if $\F$ is the class of polynomial-time forecasters, we would like a notion of sequences which are algorithmically random with respect to polynomial-time forecasters.

The authors do not yet know of a satisfactory definition of algorithmic randomness with respect to resource constraints. However, the obvious analog of Martin-L\"{o}f's original definition \citep{Martin:1966} is that a sequence should be defined as algorithmically random with respect to a class of bounded experts if, and only if, it satisfies all properties that hold of randomly generated sequences with probability~1 \emph{and that can be checked by one of those experts}. On sequences that are algorithmically random with respect to $\F$ in this sense, \Lem{bounded} and \Lem{infinite} must apply: Assume $f_z$ is a Bayes-optimal predictor on a subsequence that is algorithmically random with respect to $\F$; any forecaster $f_j \in \F$ that outperforms $f_z$ infinitely often would be identifying a way in which the sequence fails to satisfy a property that randomly generated sequences satisfy with probability~1, which contradicts the assumption. This gives strong reason to expect that \EvOp would be eventually optimal when predicting sequences that are algorithmically random with respect to $\F$, even though formalizing such a notion remains an open problem.

Even so, this does not mean that \EvOp would perform \emph{well} at the actual task of predicting large computations from the observation of small ones. Eventual optimality provides no guarantees about the ability of the algorithm to converge on good but non-optimal predictors, and the bounds that we have on how long it takes \EvOp to converge on good behavior are weak (to say the least).

Furthermore, there are other notions of what it means to ``predict computations well" that are not captured by eventual optimality. For example, \citet{Demski:2012a} discusses the problem of computably assigning probabilities to the outputs of computations and refining them in such a way that they are ``coherent," drawing on inspiration from the field of mathematical logic that dates at least back to \citet{Gaifman:1964}. The intuition is that given two statements ``this computation will halt and output~1" and ``this computation will fail to halt or output something besides~1," a good reasoner should assign those claims probabilities that sum to roughly 1. We have no reason to expect that \EvOp has any such property.

\section{Conclusions} \label{sec:conclusions}

We have studied online learning in a setting where delays between prediction and observation may be unbounded, in attempts to explore the general problem of predicting the behavior of large computations from observations of many small ones. We found that, in the stochastic setting, the unbounded delays give rise to difficulties: Total regret and average regret are not good measures of forecaster success, and consistency is not possible to achieve in general. However, it is possible to converge on good predictions by comparing forecasters according to their performance only on sparse and independent subsequences of the observations, and we have reason to expect that some of the techniques used to achieve good performance in the stochastic setting will carry over into the deterministic setting. We have proposed an algorithm \EvOp that converges to optimal behavior. It is not a practical algorithm, but it does give a preliminary model of online learning in the setting where the delay between prediction and feedback is ever-growing.

Our results suggest a few different paths for future research. \EvOp handles the problem of learning in the face of potentially unbounded delays by comparing forecasters only on subsequences that are potentially very sparse, and this means that it converges to optimal behavior quite slowly. Speeding up convergence without falling prey to the problems described in \Sec{problem} might prove difficult. Furthermore, \EvOp only guarantees convergence on forecasters that are Bayes-optimal; it is not yet clear how to converge on the best available forecaster (even if it is non-optimal) in the face of unbounded delays. As mentioned in \Sec{deterministic}, a formal notion of algorithmic randomness with respect to a bounded class of experts would make it easier to study the problem of using online learning to predict the behavior of large computations in a deterministic setting. \EvOp is only a first step towards a predictor that can learn to predict the behavior of large computations from the observation of small ones, and the problem seems ripe for further study.

\appendix
\section{Proof of \Lem{jessica}} \label{app:jessica}

\begin{lemma} \label{lem:jessica}
  Let $\mathcal{G}$ and $\mathcal{H}$ be sets, and let $G_1, H_1, G_2, H_2, ..., G_n, H_n$ be random
  variables forming a Markov chain (with each $G_i \in \mathcal{G}$ and $H_i \in \mathcal{H}$).
  Let there be functions $v : \mathcal{G} \rightarrow \mathbb{R}_{\geq 0}$ and
  $r : \mathcal{H} \rightarrow \mathbb{R}$, with $|r(H_i)| \leq a \sqrt{v(G_i)}$ and $\mathbb{E}[r(H_i) | G_i] \leq 0$.  Let $\lambda > 0$.  Then
  \begin{equation}P\left(\sum_{i=1}^n (r(F_i) - v(G_i)) \geq \lambda\right) \leq \exp\left(-\sfrac{2}{a^2}\lambda\right)\end{equation}
\end{lemma}
\begin{proof}
  This proof closely follows the standard proof of Azuma's inequality, given by, e.g., \citet{DasGupta:2011}.
  Let $b = \sfrac{2}{a^2}$.  Using Markov's inequality:
  \begin{align}
  \begin{split}
    P& \left(\sum_{i=1}^n (r(H_i) - v(G_i)) \geq \lambda\right) \\
    &= P\left(\exp\left(b\sum_{i=1}^n (r(H_i) - v(G_i))\right) \geq \exp\left(b\lambda\right)\right)
    \\
    &\leq \exp\left(-b\lambda\right) \mathbb{E}\left[\exp\left(b\sum_{i=1}^n (r(H_i) - v(G_i))\right)\right]
  \end{split}
  \end{align}
  To bound the expectation, we will inductively show that for all $m \leq n$,
  \begin{equation}\mathbb{E}\left[\exp\left(b\sum_{i=1}^m (r(H_i) - v(G_i))\right)\right] \leq 1\end{equation}
  When $m = 0$, this is trivial.  Otherwise:
  \begin{align}
  \begin{split}
\mathbb{E} & \left[\exp\left(b\sum_{i=1}^m (r(H_i) - v(G_i))\right)\right] \\
 &= \mathbb{E}\left[
      \exp\left(
        b\sum_{i=1}^{m-1} (r(H_i) - v(G_i))
      \right)
      \exp\left(- b v(G_m) \right)
      \mathbb{E}\left[e^{b r(H_m)} | G_m \right]
     \right] \\
    &\leq
     \mathbb{E}\left[
     	\exp\left(b\sum_{i=1}^{m-1} (r(H_i) - v(G_i))\right)
		\exp\left(- b v(G_m)\right)
     	\exp\left(b^2a^2v(G_m)/2\right)
     \right] \\
     &=
     \mathbb{E}\left[
       \exp\left(b\sum_{i=1}^{m-1} (r(H_i) - v(G_i))\right)
       \exp\left(-bv(G_m) + b v(G_m)\right)
     \right] \\
     &=
     \mathbb{E}\left[\exp\left(b\sum_{i=1}^{m-1} (r(H_i) - v(G_i))\right)\right].
  \end{split}
  \end{align}
  By the inductive assumption, this quantity is no more than 1, so the inductive argument goes through.  Using this bound on the expectation, the given upper bound or the original probability of interest follows.
\end{proof}

\section{Proof of \Lem{infinite}} \label{app:infinite}
\leminf*
\begin{proof}
  Let $1/m<\delta$. It suffices to show that with probability 1,
  \begin{equation}\lim_{n\to\infty} \RelScore[$j$, $z$, $m$] = \infty.\end{equation}
  Write $t_n$ for the number of times that $\TestSeq[$j$, $z$, $m$]$ outputs, and note that $t_n \to \infty$ as $n \to \infty$ because $f_j$ and $f_z$ disagree by more than $\delta$ infinitely often. We will show that \RelScore[$j$, $z$, $m$] is bounded below by a bound proportional to $t_n$, which means that $\RelScore[$j$, $z$, $m$]$ must diverge to infinity.

  Let $s=\TestSeq[$j$,$z$,$m$][\infty]$. Define $G_1H_1G_2H_2\ldots$, $r(H_i)$, and $v(G_i)$ as in the proof of \Lem{bounded}. Recall that $r(H_i) - v(G_i)$ is an upper bound for $\Loss(\x[i],\y[i]^z)-\Loss(\x[i],\y[i]^j)$, which means that $v(G_i)-r(H_i)$ is a lower bound for $\Loss(\x[i],\y[i]^j)-\Loss(\x[i],\y[i]^z)$. Therefore, it suffices to show that, for some $\alpha > 0$,
  \begin{equation} \lim_{N \to \infty} \PP\left(\forall n > N \colon \sum_{i=1}^{t_n} \left( v(G_i)-r(H_i)-\frac{\rho\varepsilon}{2m^2} \right) \ge \alpha t_n\right) = 1.\end{equation}

  Observe that $v(G_i) \ge \sfrac{\rho}{2m^2}$ for all $i$, so the positive $v(G_i)$ terms going against $f_j$ more than compensate for the negative $\sfrac{\rho\varepsilon}{2m^2}$ terms going in its favor. Because $\varepsilon < 1$, only a $\frac{1+\varepsilon}{2}$ portion of each $v(G_i)$ is needed to cancel out the $\sfrac{\rho\varepsilon}{2m^2}$ terms,
  \begin{multline}
    \PP\left(\sum_{i=1}^{t_n} \left(v(G_i)-r(H_i)-\frac{\rho\varepsilon}{2m^2}\right) \ge \alpha t_n \right) \\
    \ge \PP\left(\sum_{i=1}^{t_n} \left(\frac{1-\varepsilon}{2}v(G_i)-r(H_i) \right) \ge t_n\left(\alpha + \frac{\rho(\varepsilon - 1)}{4m^2}\right) \right).
 \end{multline}
Now we apply \Lem{jessica}. $\EE[r(H_k)\mid G_k]$ is still $0$. With $a = \frac{\kappa\sqrt{2}}{\sqrt{\rho}}\cdot\sqrt{\frac{2}{1-\varepsilon}}$,
  \begin{equation}|r(H_k)| \leq a\sqrt{\frac{1-\varepsilon}{2}v(G_k)}.\end{equation}
  Therefore, by \Lem{jessica} we have that
  \begin{equation}\PP\left(\sum_{i=1}^{t_n} \left(r(H_i)-\frac{1-\varepsilon}{2}v(G_i) \right)\geq M\right)\leq\exp \left(-\frac{\rho(1-\varepsilon)M}{2\kappa^2}\right).\end{equation}
  Choose $\alpha = \sfrac{\rho(1-\varepsilon)}{8m^2}$ and set $M = -t\left(\alpha + \frac{\rho(\varepsilon-1)}{4m^2}\right) = t\frac{\rho(1-\varepsilon)}{8m^2}$ to get:
  \begin{equation}
    \PP\left(\sum_{i=1}^{t_n} \left(\frac{1-\varepsilon}{2} v(G_i)-r(H_i)\right) \le t\frac{\rho(\varepsilon-1)}{8m^2}\right)
    \le \exp\left(-\frac{t\rho^2(1-\varepsilon)^2}{16m^2\kappa^2}\right).
  \end{equation}
  We write $c = \sfrac{\rho^2(1-\varepsilon)^2}{16m^2\kappa^2}$ for conciseness. Observe that
  \begin{multline}
    \PP \left(\exists n\geq N \colon \sum_{i=1}^{t_n} \left(v(G_i)-r(H_i)-\frac{\rho\varepsilon}{2m^2} \right)\le \alpha t_n\right) \\
    \le \sum_{t=t_N}^\infty\exp(-tc)
    = \frac{\exp(-t_Nc)}{1-\exp(-c)}.
  \end{multline}

  If $|s| = \infty$ then the right-hand side almost surely goes to zero as $n \to \infty$, in which case, with probability~1, there exists an $N$ such that
  \begin{equation}\forall n > N \colon \sum_{i=1}^{t_n} \left(v(G_i)-r(H_i)-\frac{\rho\varepsilon}{2m^2} \right) \ge \alpha t_n.\end{equation}
  Thus if $f_z$ and $f_j$ disagree by more than $\delta$ infinitely often, then with probability~1, eventually \RelScore[$j$,$z$,$m$] grows proportionally to $t_n.$ Therefore, with probability~1,
  \begin{equation}\lim_{n\to\infty} \RelScore[$j$,$z$,$m$] = \infty,\end{equation}
  so \MaxScore[$j$] almost surely diverges to $\infty$ as $n \to \infty.$
\end{proof}

\section{Proof of \Thm{bounds}} \label{app:bounds}
Let $f_z$ be a Bayes-optimal predictor and assume $\F$ is finite. Assume we have an increasing function $h$ such that for some $m$ and every $f_j$, for all times $t$, there exists a $t < t^\prime < h_j^m(t)$ such that $\y[t^\prime]^z \coloneqq f_z(\allobs[t^\prime])$ and $\y[t^\prime]^j \coloneqq f_j(\allobs[t^\prime])$ are both defined and $\norm{\y[t^\prime]^z - \y[t^\prime]^j} > \sfrac{1}{m}$. Assume we have an increasing function $g$ such that $\allobs[g(t)](t)$ is always defined. $\circ$ denotes function composition; i.e., $(h\circ g)^n(1)$ denotes $h(g(\ldots h(g(1))))$ with $n$ calls to $h$ and $g$.

\thmbounds*

\begin{proof}
Observe that $\TestSeq[$j$, $z$, $m$]$ outputs at least $t$ terms for some $t$ such that $(h\circ g)^t(1)\leq n$. In the proof of \Lem{bounded}, we prove that the probability that $\MaxScore[z]\geq \lambda$ for any $n$ is at most
\begin{equation}\frac{\exp(-b(\lambda+2-z))}{(1-\exp(-b\rho\varepsilon/2))(1-\exp(-b))^2}.\end{equation}
In the proof of \Lem{infinite}, we prove that the probability that \begin{equation}\MaxScore[$j$]\leq \alpha t-m-z+j\end{equation} for any $n$ such that $\TestSeq[$j$, $z$, $m$]$ outputs at least $t$ terms is at most
\begin{equation}\frac{\exp(-tc)}{1-\exp(-c)}.\end{equation}
Combining these, we get that for any $T$, if we let $t$ be the maximal $t$ such that $(h\circ g)^t(1)\le T$,
then for $\lambda=\alpha t-m-z+|\F|,$
with probability at least
\begin{equation} \label{eq:b1}
1-\left(\frac{\exp(-b(\lambda+2-z))}{(1-\exp(-b\rho\varepsilon/2))(1-\exp(-b))^2}+|\F|\frac{\exp(-tc)}{1-\exp(-c)}\right),
\end{equation}
$\EvOp[{\allobs[n]}]=f_z(\allobs[n])$ for all times after $T$. This also gives us a weak bound on total loss: Because \Loss is both Lipschitz and strongly convex, it is bounded. Let $L$ be the bound. Then with probability as per \Eqn{b1}, the total loss never goes above $LT$.

Reversing this process, we also get that for any $p$, if we let $t$ be such that
\begin{equation} \label{eq:b2}
\left(\frac{\exp(-b(\alpha t-m-z+|\F|+2-z))}{(1-\exp(-b\rho\varepsilon/2))(1-\exp(-b))^2}+|\F|\frac{\exp(-ct)}{1-\exp(-c)}\right)<p,
\end{equation}
then with probability at least $1-p,$ for all $n\geq (h\circ g)^t(1)$, $\EvOp[{\allobs[n]}]=f_z(\allobs[n])$.
\end{proof}

\section*{Acknowledgements}
Thanks to Jessica Taylor for the proof of \Lem{jessica}, and to Benya Fallenstein for helpful discussions.

This research was supported as part of the Future of Life Institute (futureoflife.org) FLI-RFP-AI1 program, grant~\#2015-144576.

\printbibliography

\end{document}